%% file: 0.main.tex
\pdfoutput=1

\documentclass[11pt]{article}

\usepackage{EACL2023}
\usepackage{times}
\usepackage{latexsym}

\usepackage[T1]{fontenc}

\usepackage[utf8]{inputenc}

\usepackage{microtype}

\usepackage{inconsolata}

\title{On the inconsistency of separable losses for structured prediction}

\author{Caio Corro\\
Universite Paris-Saclay, CNRS, LISN, 91400, Orsay, France\\
\texttt{caio.corro@lisn.upsaclay.fr}}

\input{math_commands}

\usepackage{amsmath}
\usepackage{bbm}
\usepackage{verbatim}

\usepackage{amsthm}
\newtheorem{definition}{Definition}
\newtheorem{theorem}{Theorem}

\usepackage[cal=boondoxo]{mathalpha}
\newcommand\risk{\mathcal{r}}

\usepackage{tikz}
\usetikzlibrary{decorations.pathreplacing,calligraphy,matrix,positioning}

\begin{document}
\maketitle
\begin{abstract}
In this paper, we prove that separable negative log-likelihood losses for structured prediction are not necessarily Bayes consistent, or, in other words, minimizing these losses may not result in a model that predicts the most probable structure in the data distribution for a given input.
This fact opens the question of whether these losses are well-adapted for structured prediction and, if so, why.
\end{abstract}

\input{1.introduction}
\input{2.bayes_consistency}
\input{3.nll}

\input{4.ner}
\input{5.dep_parsing}
\input{6.conclusion}

\section*{Acknowledgments}

We thank Joseph Le Roux, François Yvon and the anonymous reviewers for their comments and suggestions.

\bibliography{refs}
\bibliographystyle{acl_natbib}

\cleardoublepage
\appendix
\input{7.appendix}

\end{document}

%% file: math_commands.tex
\usepackage{amsmath,amsfonts,bm}

\def\eqref#1{equation~\ref{#1}}

\def\1{\bm{1}}

\def\rvx{{\mathbf{x}}}
\def\rvy{{\mathbf{y}}}

\def\ervy{{\textnormal{y}}}

\def\va{{\bm{a}}}
\def\vb{{\bm{b}}}
\def\vc{{\bm{c}}}
\def\vd{{\bm{d}}}
\def\ve{{\bm{e}}}
\def\vf{{\bm{f}}}
\def\vg{{\bm{g}}}
\def\vh{{\bm{h}}}
\def\vi{{\bm{i}}}

\def\vw{{\bm{w}}}
\def\vx{{\bm{x}}}
\def\vy{{\bm{y}}}

\def\evw{{w}}

\def\evy{{y}}

\DeclareMathAlphabet{\mathsfit}{\encodingdefault}{\sfdefault}{m}{sl}
\SetMathAlphabet{\mathsfit}{bold}{\encodingdefault}{\sfdefault}{bx}{n}

\newcommand{\E}{\mathbb{E}}

\newcommand{\R}{\mathbb{R}}

\DeclareMathOperator*{\argmax}{arg\,max}
\DeclareMathOperator*{\argmin}{arg\,min}

%% file: 1.introduction.tex
\section{Introduction}
\label{sec:intro}

Modern natural language processing (NLP) heavily relies on machine learning (ML), where prediction models are learned by minimizing a loss function over the training data.
As such, loss functions play a central role in the design of these systems and it is important to understand their statistical properties in order to guarantee that the corresponding training objectives are well defined.
Although this topic is well studied in the ML community \cite{gabor2004consistency,lin2004fisher,zhang2004statistical,zhang2004,bartlett2006convexity,tilmann2007,liu07consistency,tewari07consistency,reid10,williamson2016,duchi2018,blondel2020fy,nowak22maxmargin}, \emph{inter alia}, there has been less focus on the structured prediction setting apart from a few recent works \cite{blondel2019projection,nowak2019general,nowak20maxminmargin}.

In this paper, we emphasize the fact that, despite achievements in terms of accuracy, statistical behavior of loss functions used in practice for structured prediction in NLP are not always well understood.
We illustrate this fact by proving that commonly used separable loss functions for named entity recognition (NER) and dependency parsing are not Bayes consistent, meaning that training a model with these loss functions will not necessarily result in the prediction of the most the probable output for a given input in the data distribution.

%% file: 2.bayes_consistency.tex
\section{Bayes consistency}

We denote inputs and outputs as $\vx \in X$ and $\vy \in Y$, respectively.
We assume each $\vy \in Y$ is a binary vector whose elements are indexed by a set $C$, \emph{i.e.}\ $\vy \in \{0, 1\}^C$, where $C$ is problem dependent.
For example, in the $k$ multiclass classification case, we have $C \ = [k]$, where we use the shorthand $[k] = \{1,2,...,k\}$, and $Y$ is defined as the set of standard bases (one-hot vectors) of dimension $k$, meaning that $|Y| = k$.
More generally, the vector $\vy$ is an indicator of ``selected'' parts in $C$ and, in the  structured prediction case, several parts can be jointly selected.
Note that it is usual to assume that the parts in $C$ can depend on the input $\vx$.
Without loss of generality, we omit this detail as we will study loss functions in the pointwise setting.

A scoring model $f \in F$ is a function $f: X \to \R^C$ that returns scores associated with each part in $C$ for a given input, \emph{e.g.}\ the score of each class in a multiclass classification model.
The actual prediction of the model is the output of maximum linear score:
\begin{align}
    \label{eq:map}
    \widehat\vy(\vx) \in \argmax_{\vy \in Y}~\langle \vy, f(\vx) \rangle,
\end{align}
where $\langle \cdot, \cdot \rangle$ denotes the inner product.
We refer to computing Equation~\ref{eq:map} as maximum \emph{a posteriori} (MAP) inference.

A loss function compares a vector of scores with an expected output.
Importantly, the 0-1 loss function is defined as follows:
\begin{align*}
    \ell(\vw, \vy) = \begin{cases}
        0\quad&\text{if } \vy \in \argmax_{\vy' \in Y} \langle \vw, \vy' \rangle, \\
        1\quad&\text{otherwise,}
    \end{cases}
\end{align*}
where $\vw \in \R^k$ is a vector of part scores, \emph{i.e.}\ $\vw = f(\vx)$ for a given input $\vx$.
In order to choose a scoring function $f \in F$, it is appealing to select one that minimizes this loss over the data distribution:
\begin{align*}
    \risk^*
    &
    = \inf_{f \in F} \risk(f)
    = \inf_{f \in F} \E_{\rvx, \rvy}[\ell(f(\rvx), \rvy)],
\end{align*}
where $\rvx$ and $\rvy$ are random variables over inputs and outputs, respectively, and $\inf$ denotes the infimum.
The value $\risk(f)$ is the Bayes risk of function $f$ and $\risk^*$ is the optimal Bayes risk.
For theoretical purposes, it is often assumed that the class of functions $F$ is rich enough (the set of all measurable mappings) to obtain the best possible risk.
Then, the optimal Bayes risk is equal to:
\begin{align*}
    \risk^*
    &= \E_\rvx[1 - \max_{\vy \in Y} p(\rvy = \vy | \rvx)],
\end{align*}
or, in other words, it is the probability of making an error when the classifier predicts the most probable class for each input.

Unfortunately, in practice it is not convenient to use the 0-1 loss $\ell$ as it is nonconvex and has null derivatives almost everywhere.
Instead, a surrogate $\widetilde{\ell}$ can be used as a loss function:
\begin{align*}
    \widetilde{\risk}^*
    &
    = \inf_{f \in F}~\widetilde{\risk}(f)
    = \inf_{f \in F} \E_{\rvx, \rvy}[\widetilde{\ell}(f(\rvx), \rvy)],
\end{align*}
where $\widetilde{r}(f)$ is the surrogate risk of function $f$ and $\widetilde{\risk}^*$ is the optimal surrogate risk.
An important desired property of surrogate losses is their consistency with the 0-1 loss, \emph{i.e.}\ the fact that minimizing the surrrogate risk leads to a prediction model of optimal Bayes risk \cite{gabor2004consistency,lin2004fisher,zhang2004statistical,bartlett2006convexity,liu07consistency,tewari07consistency}.

\begin{definition}

A surrogate loss $\widetilde{\ell}$ is said to be Bayes consistent\footnote{This property is also referred to as Fisher consistency \cite{lin2004fisher,bartlett2006convexity,liu07consistency} and classification calibration \cite{williamson2016}.} if:
\begin{align*}
    f^* \in \argmin_{f \in F}~\widetilde{\risk}(f) \implies \risk(f^*) = \risk^*.
\end{align*}
\end{definition}

Note that this property can be checked independently for each input $\vx$ (called pointwise Bayes consistency) as we assume a rich enough class of functions $F$.
In other words, we redefine the pointwise (optimal) surrogate risk as:
\begin{align*}
    \widetilde{\risk}^*
    &
    = \inf_{\vw \in \R^C}~\widetilde{\risk}(\vw)
    = \inf_{\vw \in \R^C} \E_{\rvy | \rvx = \vx}[\widetilde{\ell}(\vw, \rvy)],
\end{align*}
for any $\vx$ such that $p(\rvx = \vx) > 0$, and similarly for the optimal Bayes risk.
The vector $\vw$ should be interpreted as the model scores, \emph{i.e.}\ $\vw = f(\vx)$.

%% file: 3.nll.tex
\section{Negative log-likelihood loss}

The negative log-likelihood loss (NLL), also known as the conditional random field loss \cite{lafferty2001crf}, is defined as follows:
\begin{align*}
    \widetilde{\ell}_{(nll)}(\vw, \vy)
    &= - \langle \vw, \vy \rangle + \log \sum_{\vy' \in Y} \exp\langle \vw, \vy' \rangle.
\end{align*}
In the following, we will refer to computing the log-sum-exp term of the NLL loss as marginal inference due to its connection with marginal probabilities \cite{wainwright2008exp}.

\begin{theorem}
\label{th:nll}
Under mild conditions on the data distribution, the surrogate loss $\widetilde{\ell}_{(nll)}$ is Bayes consistent.
\end{theorem}

\begin{proof}
The optimal pointwise surrogate Bayes risk is defined as:
\begin{align*}
    \widetilde{\risk}^*_{(nll)}
    &= \inf_{\vw \in \R^C}
    \begin{array}[t]{l}
        -\E_{\rvy | \rvx = \vx}[\langle \vw, \rvy \rangle] \\
        + \log \sum_{\vy' \in Y} \exp \langle \vw, \vy' \rangle.
    \end{array}
    \\
\intertext{We substitute $\vw(\vy) = \langle \vw, \vy \rangle$ for all $\vy \in Y$:}
    &= \inf_{\substack{\forall \vy \in Y: \\ \vw(\vy) \in \R}}
    \begin{array}[t]{l}
        - \E_{\rvy|\rvx = \vx}[\vw(\rvy)] \\
        + \log \sum_{\vy' \in Y} \exp \vw(\vy').
    \end{array}
\end{align*}
We denote $\widehat\vw(\vy), \forall \vy \in Y$, an optimal solution of the minimization.
By first order optimality conditions, we have:
\begin{align}
    \nonumber
    &\frac{\partial}{\partial \widehat\vw(\vy)} \left(
        \begin{array}[l]{l}
            - \E_{\rvy|\rvx = \vx}[\widehat\vw(\rvy)] \\
            + \log \sum_{\vy' \in Y} \exp(\widehat\vw(\vy'))
        \end{array}
    \right) = 0
    \\
    &\implies
    \frac{\exp \widehat\vw(\vy)}{\sum_{\vy' \in Y} \exp\widehat\vw(\vy')} = p(\rvy = \vy | \rvx = \vx)
    \label{eq:nll_bayes}
\end{align}
which implies Bayes consistency under the condition that there exists a vector $\vw \in \R^C$ such that $\forall \vy \in Y: \langle \vw, \vy \rangle = \widehat\vw(\vy)$.
\end{proof}

To understand why Equation~\ref{eq:nll_bayes} implies Bayes consistency, note that:
\begin{align*}
    \frac{\exp \widehat\vw(\vy)}{\sum_{\vy' \in Y} \exp \widehat\vw(\vy')}
    \propto \exp \widehat\vw(\vy'),
\end{align*}
and the exponential function is strictly increasing.
This means scores of outputs $\vy \in Y$ defined as $\langle \vy, \widehat\vw \rangle$ are ordered in the same way as probabilities in the data distribution $p(\rvy|\rvx = \vx)$.
In other words, the most probable output in the data distribution will have the highest score with respect to $\widehat\vw$.

The proof is a straightforward extension of the derivation for the multiclass classification case, see for example \citet[Section 4.2]{blondel2020fy}.
Interestingly, Equation~\ref{eq:nll_bayes} also implies that the NLL loss is strictly proper \cite{williamson2016}, \emph{i.e.}\ the Boltzmann distribution over structures in $Y$ parameterized by minimizer $\widehat\vw$ is equal to the data distribution $p(\rvy|\rvx = \vx)$.
Note that Theorem~\ref{th:nll} is not novel \emph{per se} and a more in-depth study of NLL losses for structured prediction can be found in \cite{nowak2019general}.

One limitation of the NLL loss is that it is not (additively) separable\footnote{A function $f$ is additively separable if it can be written as $f = \sum_i f_i.$} because of the log-sum-exp term.
As such, this term is a bottleneck for parallel computation of the objective and doubly stochastic estimation of the training objective \cite{titsias2016onevseach}.
A well-known solution is to rely on independent binary classification objectives, also known as one-vs-all losses \cite[Section 6.1]{blondel2020fy}:
\begin{align*}
    \widetilde{\ell}_{(one-vs-all)}(\vw, \vy)
    &\\
    &\hspace{-2.5cm}= - \langle \vw, \vy \rangle + \sum_{\vy' \in Y} \log \left(1 +  \exp\langle \vw, \vy' \rangle\right).
\end{align*}
In the case of multiclass classification problems, it can be shown that this objective is Bayes consistent using similar arguments as in Theorem~\ref{th:nll}.
A different approach is the one-vs-each loss function that is also Bayes consistent \cite{titsias2016onevseach}.

Unfortunately, these separable surrogates cannot be applied to structured prediction problems as the set $Y$ is often of exponential size with respect to the input length.
Although tractable algorithms for marginal inference exist for many cases, there are no known algorithms to compute the one-vs-each or one-vs-all losses in an easily parallelizable fashion.
As such, the NLP community often relies on token-separable losses, that is a NLL objective that decomposes as a sum of independent losses, one per token in the input sentence.
Although these losses are easy to implement, we prove in the next sections that they are not Bayes consistent for two common NLP problems.

%% file: 4.ner.tex
\section{Named-entity recognition}

\textbf{Problem definition.}
In this Section, we focus on the flat NER problem using BIO tags \cite{ratinov2009bilou}.
Without loss of generality, we assume there is a single mention label and that the input sentence $\vx$ contains $n$ words.
The set of parts is defined as $C = [n] \times \{\textsc{B}, \textsc{I}, \textsc{O}\}$ and $Y$ is defined as the set of vectors $\vy \in \{0, 1\}^C$ satisfying the following conditions:
\begin{enumerate}
    \item $\forall i \in [n]: \sum_{t} \evy_{i, t} = 1$ (one tag per word);
    \item $\evy_{1, \textsc{I}} = 0$ (forbid inside tag for the first word);
    \item $\forall i > 1: \evy_{i, \textsc{I}} = 1 \implies \evy_{i-1, \textsc{B}} + \evy_{i-1, \textsc{I}} = 1$ (\textsc{I} tag can only follow a \textsc{B} or \textsc{I} tag).
\end{enumerate}
We do not include parts corresponding to transitions (this is a unigram model), otherwise it would not be possible to derive a token-separable loss.

\textbf{Inference algorithms.}
MAP and marginal inference can be realized using the Viterbi and the forward-backward algorithms, respectively.
Although the time complexity of these algorithms is $\mathcal O(|L|^2n)$ where $L$ is the set of mention labels, they can be optimized to have a $\mathcal O(|L|n)$ time complexity as there is no transition score.
The dynamic programming algorithm is nonetheless required in order to guarantee that condition (3) is satisfied.

\textbf{Separable loss.}
As the dynamic programming algorithm is not parallelizable over input tokens, token-separable losses are often used in practice.\footnote{See for example \url{https://github.com/huggingface/transformers/blob/v4.23.1/src/transformers/models/bert/modeling_bert.py\#L1771}}
That is, the loss is reduced to a set of $n$ multiclass classification losses:
\begin{align*}
    \widetilde{\ell}_{\text{(sep-bio)}}
    = - \langle \vw, \vy \rangle
    + \sum_{i = 1}^n \log \sum_{t} \exp \evw_{i, t},
\end{align*}
where $t$ ranges over all tags, except \textsc{I} if $i=1$.

The optimal pointwise surrogate Bayes risk for the separable loss is defined as:
\begin{align*}
    \widetilde{\risk}^*_{(sep-bio)}
    = \inf_{\vw \in \R^C} 
    \begin{array}[t]{l}
    - \E_{\rvy|\rvx = \vx}[\langle \vw, \rvy \rangle] \\
     + \sum_{i = 1}^n \log \sum_{t} \exp \evw_{i, t}.
     \end{array}
\end{align*}
Let $\widehat\vw$ be an optimal solution.
Then, by first order optimality conditions:
\begin{align}
    & \nonumber
    \frac{\partial}{\partial \widehat\evw_{i, t}}
    \left(\begin{array}{l}
    - \E_{\rvy|\rvx = \vx}[\langle \vw, \rvy \rangle] \\
     + \sum_{i = 1}^n \log \sum_{t} \exp \evw_{i, t}
     \end{array}\right)
    = 0
    \\
    & \label{eq:unstruct_bio}
    \implies \widehat\evw_{i, t} = \log p\left(\ervy_{i, t} = 1 | \rvx = \vx\right)
\end{align}
where $p\left(\ervy_{i, t} = 1 | \rvx = \vx\right)$ denotes the marginal distribution of tag $t$ at position $i$ in data distribution.

\input{figures/bio}

\begin{theorem}
The token-separable loss for NER via BIO tagging is not Bayes consistent.
\end{theorem}

\begin{proof}
Let $n=2$ and assume that the distribution $p(\rvy | \rvx = \vx)$ is defined as depicted in Figure~\ref{distrib:ner}.
Then, by Equation~\ref{eq:unstruct_bio} we have $\widehat\evw_{1, \textsc{B}} = \log 0.65$, $\widehat\evw_{1, \textsc{O}} = \log 0.35$, $\widehat\evw_{2, \textsc{B}} = \log 0.4$, $\widehat\evw_{2, \textsc{I}} = \log 0.3$ and $\widehat\evw_{2, \textsc{O}} = \log 0.3$.
As such:
\begin{align*}
    \log 0.65 + \log 0.3 &= \langle \widehat\vw, \va \rangle \\
    &< \langle \widehat\vw, \vb \rangle = \log 0.65 + \log 0.4,
\end{align*}
but $\va \in \argmax_{\vy \in Y} p(\rvy = \vy | \rvx = \vx)$ and $p(\rvy = \va | \rvx = \vx) > p(\rvy = \vb | \rvx = \vx)$.
Therefore, the token-separable loss is not Bayes consistent for NER, \emph{i.e.}\ a scoring model minimizing the surrogate risk may not lead to predicting the most probable output in the data distribution.
\end{proof}

Note that the inconsistency is not due to the fact that the parameterization of the model is ``poor'' (no transition scores).
Indeed, by Equation~\ref{eq:nll_bayes}, optimal scores $\widehat\vw$ for the NLL loss satisfy the following condition for all $\vy \in Y$:
\begin{align*}
    &\frac{\exp \langle \widehat\vw, \vy\rangle}{\sum_{\vy' \in Y} \exp \langle \widehat\vw, \vy' \rangle} = p(\rvy = \vy | \rvx = \vx) \\
    &\implies \langle \widehat\vw, \vy \rangle = \log p(\rvy = \vy | \rvx = \vx)
\end{align*}
The following assignment for $\widehat\vw$ satisfies this condition:
$\widehat\evw_{1, \textsc{B}} = 0$,
$\widehat\evw_{1, \textsc{O}} = 0$,
$\widehat\evw_{2, \textsc{B}} = \log 0.2$,
$\widehat\evw_{2, \textsc{I}} = \log 0.3,$ and
$\widehat\evw_{2, \textsc{O}} = \log 0.15$.
That is, minimizing the NLL loss on this distribution results in a Bayes consistent classifier, as expected.

%% file: figures/bio.tex
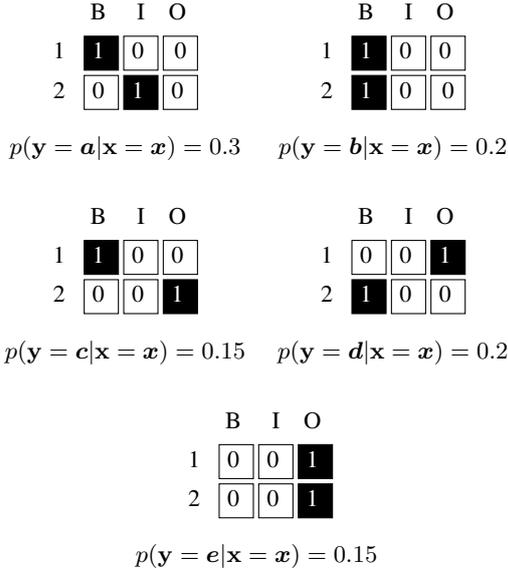
\begin{figure}[t!]
    \centering
    \begin{tabular}{cc}
         \input{figures/bio_a}
         & \input{figures/bio_b}
         \\
         {\small $p(\rvy = \va | \rvx = \vx) = 0.3$}
         & {\small $p(\rvy = \vb | \rvx = \vx) = 0.2$}
         \\[1em]
         \input{figures/bio_c}
         & \input{figures/bio_d}
         \\
         {\small $p(\rvy = \vc | \rvx = \vx) = 0.15$}
         & {\small $p(\rvy = \vd | \rvx = \vx) = 0.2$}
         \\[1em]
         \multicolumn{2}{c}{\input{figures/bio_e}}
         \\
         \multicolumn{2}{c}{{\small $p(\rvy = \ve | \rvx = \vx) = 0.15$}}
    \end{tabular}

    \caption{
        Example of distribution over BIO sequences for a sentence of 2 words. The set of sequences is defined as $Y = \{ \va, \vb, \vc, \vd, \ve \}.$
        The matrices represent values in elements of $Y$.
    }
    \label{distrib:ner}
\end{figure}

%% file: figures/bio_a.tex
\begin{tikzpicture}
\small
    \matrix (m) [matrix of nodes,
                 nodes in empty cells,
                 nodes={
                    rectangle,
                    draw=black,
                    align=center,
                    minimum height=4.5mm,
                    minimum width=4.5mm,
                    text depth=0ex,
                    text height=1ex,
                    inner xsep=0pt,
                    outer sep=0pt
                 },
                 column sep=-\pgflinewidth+2,
                 row sep=-\pgflinewidth+2,
                 ]
    {
        |[draw=none]|   &|[draw=none]|  \textsc{B}  &|[draw=none]| \textsc{I}&|[draw=none]| \textsc{O} \\
        |[draw=none]|1 & |[fill=black,text=white]| 1  &0 &  0\\
        |[draw=none]|2 & 0  & |[fill=black,text=white]| 1 & 0 \\
    };
    \end{tikzpicture}

%% file: figures/bio_b.tex
\begin{tikzpicture}
\small
    \matrix (m) [matrix of nodes,
                 nodes in empty cells,
                 nodes={
                    rectangle,
                    draw=black,
                    align=center,
                    minimum height=4.5mm,
                    minimum width=4.5mm,
                    text depth=0ex,
                    text height=1ex,
                    inner xsep=0pt,
                    outer sep=0pt
                 },
                 column sep=-\pgflinewidth+2,
                 row sep=-\pgflinewidth+2,
                 ]
    {
        |[draw=none]|   &|[draw=none]| \textsc{B}  &|[draw=none]| \textsc{I}&|[draw=none]| \textsc{O} \\
        |[draw=none]|1 & |[fill=black,text=white]| 1 &0 &0  \\
        |[draw=none]|2 &  |[fill=black,text=white]| 1 &0 &  0\\
    };
    \end{tikzpicture}

%% file: figures/bio_c.tex
\begin{tikzpicture}
\small
    \matrix (m) [matrix of nodes,
                 nodes in empty cells,
                 nodes={
                    rectangle,
                    draw=black,
                    align=center,
                    minimum height=4.5mm,
                    minimum width=4.5mm,
                    text depth=0ex,
                    text height=1ex,
                    inner xsep=0pt,
                    outer sep=0pt
                 },
                 column sep=-\pgflinewidth+2,
                 row sep=-\pgflinewidth+2,
                 ]
    {
        |[draw=none]|   &|[draw=none]| \textsc{B}  &|[draw=none]| \textsc{I}&|[draw=none]| \textsc{O} \\
        |[draw=none]|1 & |[fill=black,text=white]| 1  & 0 & 0 \\
        |[draw=none]|2 & 0 &0 & |[fill=black,text=white]| 1  \\
    };
    \end{tikzpicture}

%% file: figures/bio_d.tex
\begin{tikzpicture}
\small
    \matrix (m) [matrix of nodes,
                 nodes in empty cells,
                 nodes={
                    rectangle,
                    draw=black,
                    align=center,
                    minimum height=4.5mm,
                    minimum width=4.5mm,
                    text depth=0ex,
                    text height=1ex,
                    inner xsep=0pt,
                    outer sep=0pt
                 },
                 column sep=-\pgflinewidth+2,
                 row sep=-\pgflinewidth+2,
                 ]
    {
        |[draw=none]|   &|[draw=none]| \textsc{B}  &|[draw=none]| \textsc{I}&|[draw=none]| \textsc{O} \\
        |[draw=none]|1 & 0 &0 & |[fill=black,text=white]| 1 \\
        |[draw=none]|2 &  |[fill=black,text=white]| 1 & 0 & 0 \\
    };
    \end{tikzpicture}

%% file: figures/bio_e.tex
\begin{tikzpicture}
\small
    \matrix (m) [matrix of nodes,
                 nodes in empty cells,
                 nodes={
                    rectangle,
                    draw=black,
                    align=center,
                    minimum height=4.5mm,
                    minimum width=4.5mm,
                    text depth=0ex,
                    text height=1ex,
                    inner xsep=0pt,
                    outer sep=0pt
                 },
                 column sep=-\pgflinewidth+2,
                 row sep=-\pgflinewidth+2,
                 ]
    {
        |[draw=none]|   &|[draw=none]| \textsc{B}  &|[draw=none]| \textsc{I}&|[draw=none]| \textsc{O} \\
        |[draw=none]|1 & 0 &0 & |[fill=black,text=white]| 1 \\
        |[draw=none]|2 &  0 & 0 &|[fill=black,text=white]| 1 \\
    };
    \end{tikzpicture}

%% file: 5.dep_parsing.tex
\section{Syntactic dependency parsing}

\textbf{Problem definition.}
We consider a sentence of $n$ words and, without loss of generality, restrict to the unlabeled case to simplify notations.
In dependency parsing, the set of parts is defined as the set of possible bilexical dependencies between words, including a fake root at position $0$ used to identify root word(s) of the sentence, \emph{i.e.}\ $C = \{(h, m) \in \{0, 1,...,n\}\times [n] | h \neq m \}$, where $(h, m)$ denotes a dependency with the $h$-th word as head and the $m$-th word as modifier.
The set $Y$ is restricted to vectors $\vy \in \{0, 1\}^C$ that can be interpreted as forming a $0$-rooted spanning arborescence where words are vertices and dependencies are arcs \cite{mcdonald2005msa}.
In some cases, \emph{e.g.}\ the Universal Dependency format, it is required that the fake root position has a single outgoing arc.\footnote{\url{https://universaldependencies.org/u/overview/syntax.html}}

\textbf{Inference algorithms.}
MAP inference can be realized via the maximum spanning arborescence algorithm, which has a $\mathcal O(n^2)$ time complexity \cite{chu1965msa,edmonds1967msa,tarjan1977msa}.
The single root constraint can be taken into account using the same algorithm via the big-$M$ trick \cite[][Section 2]{fischetti1992additive}.
Marginal inference can be realized via the matrix tree theorem \cite[MTT,][]{koo2007mtt,mcdonald2007mtt,smith2007mtt}, which has $\mathcal O(n^3)$ time complexity.

\textbf{Separable loss.}
The cubic-time complexity of MTT may be prohibitive  in practice for training a model.
Moreover, the MTT relies on a computationally unstable matrix inversion and is arguably non-trivial to implement.
Hence, there has been interest in using simpler token-separable NLL loss functions \cite{zhang2017head}:
\begin{align*}
    &\widetilde{\ell}_{(sep-dep)}(\vw, \vy) \\
    &\hspace{+0.8cm}=- \langle \vw, \vy \rangle
     + \sum_{m \in [n]} \log \hspace{-7pt} \sum_{h \in [n] \setminus \{m\}} \hspace{-7pt} \exp \evw_{h, m},
\end{align*}
also called head selection loss.
This loss is a sum of multiclass classification NLL losses, one per word in the sentence, and is therefore token-separable.
As such, it can be efficiently parallelized on GPU and is trivial to implement in any ML framework.

The optimal pointwise surrogate Bayes risk for the token-separable loss is defined as:
\begin{align*}
    \widetilde{\risk}^*_{(sep-dep)}
    = \inf_{\vw \in \R^A} 
    \begin{array}[t]{l}
    - \E_{\rvy|\rvx = \vx}[\langle \vw, \rvy \rangle] \\
     + \sum\limits_{m \in [n]} \log \hspace{-7pt} \sum\limits_{h \in [n] \setminus \{m\}} \hspace{-7pt} \exp \evw_{h, m}.
     \end{array}
\end{align*}
Let $\widehat\vw$ be an optimal solution.
Then, by first order optimality conditions:
\begin{align}
    & \nonumber
    \frac{\partial}{\partial \widehat\evw_{h, m}}
    \left(\begin{array}{l}
        - \E_{\rvy|\rvx = \vx}[\langle \widehat\vw, \rvy \rangle] \\
        + \sum\limits_{m \in [n]} \log \hspace{-7pt} \sum\limits_{h \in [n] \setminus \{m\}} \hspace{-7pt} \exp \widehat\evw_{h, m}
    \end{array}\right)
    = 0
    \\
    & \label{eq:unstruct_opt}
    \implies \widehat\evw_{h, m} = \log p\left(\ervy_{h, m} = 1 | \rvx = \vx\right)
\end{align}
where $p\left(\ervy_{h, m} = 1 | \rvx = \vx\right)$ denotes the marginal distribution of the dependency between words at position $h$ and $m$, \emph{i.e.}\ the sum of the conditional probability of trees this dependency appears in.

\input{figures/multiroot}

\begin{theorem}

The loss $\widetilde{\ell}_{\text{(sep-dep)}}$ is not Bayes consistent for distributions over dependency trees.
\end{theorem}
\begin{proof}
Let $n=2$ and assume that the distribution $p(\rvy | \rvx = \vx)$ is defined as depicted in Figure~\ref{distrib:single}.
Then, by Equation~\ref{eq:unstruct_opt} we have: $\widehat\evw_{ 0, 1 } = \log 0.7$, $\widehat\evw_{ 0, 2 } = \log 0.6$, $\widehat\evw_{ 1, 2 } = \log 0.4$ and $\widehat\evw_{ 2, 1 } = \log 0.3$.

\noindent As such:
\begin{align*}
    \log 0.7 + \log 0.4 &= \langle \widehat\vw, \va \rangle \\
    &< \langle \widehat\vw, \vb \rangle = \log 0.7 + \log 0.6,
\end{align*}
but $\va \in \argmax_{\vy \in Y} p(\rvy = \vy | \rvx = \vx)$ and $p(\rvy = \va | \rvx = \vx) > p(\rvy = \vb | \rvx = \vx)$.
Therefore, the token-separable loss is not Bayes consistent, \emph{i.e.}\ a scoring model minimizing the surrogate risk may not lead to predicting the most probable tree in the data distribution.
\end{proof}

Note that using the (non-separable) NLL loss will lead to a Bayes consistent model on this distribution.
Indeed, by Equation~\ref{eq:nll_bayes}, optimal scores $\widehat\vw$ satisfy the following condition for all $\vy \in Y$:
\begin{align*}
    &\frac{\exp \langle \widehat\vw, \vy\rangle}{\sum_{\vy' \in Y} \exp \langle \widehat\vw, \vy' \rangle} = p(\rvy = \vy | \rvx = \vx) \\
    &\implies \langle \widehat\vw, \vy \rangle = \log p(\rvy = \vy | \rvx = \vx)
\end{align*}
The following assignment for $\widehat\vw$ satisfies this condition:
$\widehat\evw_{0, 1} = 0$,
$\widehat\evw_{0, 2} = \log 0.3$,
$\widehat\evw_{1, 2} = \log 0.4$ and
$\widehat\evw_{2, 1} = 0$.
That is, minimizing the NLL loss on this distribution results in a Bayes consistent classifier, as expected.

The single root constraint case is reported in Appendix~\ref{app:singleroot}.

%% file: figures/multiroot.tex
\begin{figure}[t!]
    \centering
    \begin{tabular}{ccc}
         \input{figures/singleroot_a} &
         \input{figures/singleroot_b}
         \\
         {\small $p(\rvy = \va | \rvx = \vx) = 0.4$}&
         {\small $p(\rvy = \vb | \rvx = \vx) = 0.3$}
         \\[0.3cm]
         \multicolumn{2}{c}{
            \input{figures/singleroot_c}
         }
         \\
         \multicolumn{2}{c}{
            {\small $p(\rvy = \vc | \rvx = \vx) = 0.3$}
         }
    \end{tabular}
    \caption{
    Example of distribution over trees for a sentence of 3 words. The set of trees is defined as $Y = \{ \va, \vb, \vc \}.$
    }
    \label{distrib:single}
\end{figure}
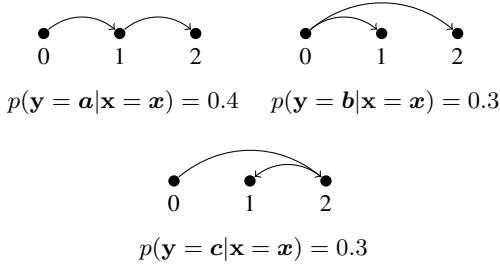

%% file: figures/singleroot_a.tex
\begin{tikzpicture}[
    scale=1,
    every node/.style={
        scale=1,
        circle,
        fill=black,
        inner sep=1.5pt
    }
]
    \node[label={below:{\small 0}}] (0) at (0, 0) {};
    \node[label={below:{\small 1}}] (1) at (1, 0) {};
    \node[label={below:{\small 2}}] (2) at (2, 0) {};
    
    \draw[->] (0) to [bend left=40] (1);
    \draw[->] (1) to [bend left=40] (2);
\end{tikzpicture}

%% file: figures/singleroot_b.tex
\begin{tikzpicture}[
    scale=1,
    every node/.style={
        scale=1,
        circle,
        fill=black,
        inner sep=1.5pt
    }
]
    \node[label={below:{\small 0}}] (0) at (0, 0) {};
    \node[label={below:{\small 1}}] (1) at (1, 0) {};
    \node[label={below:{\small 2}}] (2) at (2, 0) {};
    
    \draw[->] (0) to [bend left=40] (1);
    \draw[->] (0) to [bend left=40] (2);
\end{tikzpicture}

%% file: figures/singleroot_c.tex
\begin{tikzpicture}[
    scale=1,
    every node/.style={
        scale=1,
        circle,
        fill=black,
        inner sep=1.5pt
    }
]
    \node[label={below:{\small 0}}] (0) at (0, 0) {};
    \node[label={below:{\small 1}}] (1) at (1, 0) {};
    \node[label={below:{\small 2}}] (2) at (2, 0) {};
    
    \draw[->] (0) to [bend left=40] (2);
    \draw[->] (2) to [bend right=40] (1);
\end{tikzpicture}

%% file: 6.conclusion.tex
\section{Conclusion}

Studying statistical properties of surrogate loss functions has not been a major interest in the NLP community, although there are exceptions \cite{ma2018nce,effland2021partially}.
We proved that token-separable losses for NER and dependency parsing are not Bayes consistent, which means that minimizing these losses will not necessarily lead to models that will predict the most probable output for a given input in the data distribution, even with infinite training data.

In the dependency parsing case, \citet{zhang2020treecrf} experimentally observed that the structured NLL loss leads to better results than the token-separable head selection loss.
As such, our analysis provides a better theoretical understanding of these experiments.
However, separable losses are widely used in state-of-the-art models, which suggests that future research should study why they work in practice.

Other types of separability have also been used for constituency parsing \cite{corro2020disc} and semantic parsing \cite{pasupat2019top}, \emph{inter alia}.

\section*{Limitations}

Arguably, these separable loss functions perform well in practice, which questions the appropriateness of the Bayes consistency property.
For example, \citet{dasgupta2013hconsistency} argued that assumptions usually made are too unrealistic (\emph{e.g.}\ considering that $F$ is the set of all measurable mappings) and leads to misleading theoretical knowledge when it comes to actual implementation and experiments.
Maybe this is also the case of the demonstration we made in this paper.
However, all in all, we hope that this work will motivate future fundamental research on ML for NLP and especially on properties of loss functions.

%% file: 7.appendix.tex
\section{Inconsistency of the separable loss for single root dependency parsing}
\label{app:singleroot}

\begin{theorem}
The loss $\widetilde{\ell}_{\text{(sep-dep)}}$ is not Bayes consistent for distributions over single-root dependency trees.
\end{theorem}

\begin{proof}
Let $n=4$ and assume that the distribution $p(\rvy | \rvx = \vx)$ is defined as depicted in Figure~\ref{distrib:multi}.
Then, by Equation~\ref{eq:unstruct_opt} we have, among others:
\begin{align*}
    \widehat\evw_{ 0, 1 } &= \log 0.55, &
    \widehat\evw_{ 1, 2 } &= \log 0.55, \\
    \widehat\evw_{ 1, 3 } &= \log 0.4, &
    \widehat\evw_{ 2, 3 } &= \log 0.45.
\end{align*}
As such, we have:
\begin{align*}
    \langle \widehat\vw, \va \rangle & = \log 0.55 + \log 0.55 + \log 0.40, \\
    \langle \widehat\vw, \vb \rangle &=\log 0.55 + \log 0.55 + \log 0.45,
\end{align*}
which means that $\langle \widehat\vw, \va \rangle < \langle \widehat\vw, \vb \rangle$ but $\va \in \argmax_{\vy \in Y} p(\rvy = \vy | \rvx = \vx)$ and $p(\rvy = \va | \rvx = \vx) > p(\rvy = \vb | \rvx = \vx)$.
Therefore the separable loss is not Bayes consistent for single-root dependency parsing.
\end{proof}
\newpage
\input{figures/singleroot}
\mbox{}

%% file: figures/singleroot.tex
\begin{figure}[t!]
    \centering
\begin{tabular}{@{}cc@{}}
     \input{figures/multiroot_a} &
     \input{figures/multiroot_b}
     \\
     {\small $p(\rvy = \va | \rvx = \vx) = 0.3$} &
     {\small $p(\rvy = \vb | \rvx = \vx) = 0.2$}
     \\[15pt]
     \input{figures/multiroot_c} &
     \input{figures/multiroot_d}
     \\
     {\small $p(\rvy = \vc | \rvx = \vx) = 0.05$} &
     {\small $p(\rvy = \vd | \rvx = \vx) = 0.2$}
     \\[15pt]
     \input{figures/multiroot_e} &
     \input{figures/multiroot_f}
     \\
     {\small $p(\rvy = \ve | \rvx = \vx) = 0.05$} &
     {\small $p(\rvy = \vf | \rvx = \vx) = 0.05$}
     \\[15pt]
     \input{figures/multiroot_g} &
     \input{figures/multiroot_h}
     \\
     {\small $p(\rvy = \vg | \rvx = \vx) = 0.05$} &
     {\small $p(\rvy = \vh | \rvx = \vx) = 0.05$}
     \\[15pt]
     \multicolumn{2}{c}{\input{figures/multiroot_i}} \\
     \multicolumn{2}{c}{{\small $p(\rvy = \vi | \rvx = \vx) = 0.05$}}
\end{tabular}
    \caption{
    Example of distribution over trees for a sentence of 3 words and the single-root constraints. The set of trees is defined as $Y = \{ \va, \vb, \vc, \vd, \ve, \vf, \vg, \vh, \vi \}.$
    }
    \label{distrib:multi}
\end{figure}
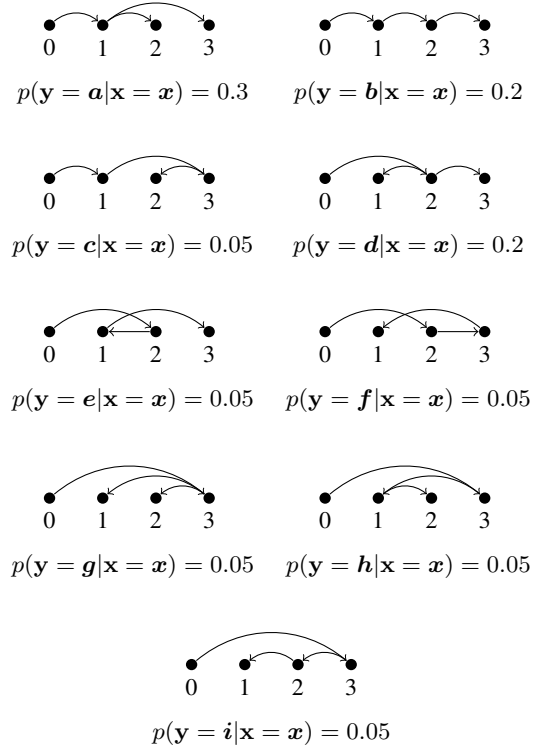

%% file: figures/multiroot_a.tex
\begin{tikzpicture}[
    every node/.style={
        circle,
        fill=black,
        inner sep=1.5pt
    }
]
    \node[label={below:{\small 0}}] (0) at (0, 0) {};
    \node[label={below:{\small 1}}] (1) at (0.7, 0) {};
    \node[label={below:{\small 2}}] (2) at (1.4, 0) {};
    \node[label={below:{\small 3}}] (3) at (2.1, 0) {};
    
    \draw[->] (0) to [bend left=40] (1);
    \draw[->] (1) to [bend left=40] (2);
    \draw[->] (1) to [bend left=40] (3);
\end{tikzpicture}

%% file: figures/multiroot_b.tex
\begin{tikzpicture}[
    every node/.style={
        circle,
        fill=black,
        inner sep=1.5pt
    }
]
    \node[label={below:{\small 0}}] (0) at (0, 0) {};
    \node[label={below:{\small 1}}] (1) at (0.7, 0) {};
    \node[label={below:{\small 2}}] (2) at (1.4, 0) {};
    \node[label={below:{\small 3}}] (3) at (2.1, 0) {};
    
    \draw[->] (0) to [bend left=40] (1);
    \draw[->] (1) to [bend left=40] (2);
    \draw[->] (2) to [bend left=40] (3);
\end{tikzpicture}

%% file: figures/multiroot_c.tex
\begin{tikzpicture}[
    every node/.style={
        circle,
        fill=black,
        inner sep=1.5pt
    }
]
    \node[label={below:{\small 0}}] (0) at (0, 0) {};
    \node[label={below:{\small 1}}] (1) at (0.7, 0) {};
    \node[label={below:{\small 2}}] (2) at (1.4, 0) {};
    \node[label={below:{\small 3}}] (3) at (2.1, 0) {};
    
    \draw[->] (0) to [bend left=40] (1);
    \draw[->] (1) to [bend left=40] (3);
    \draw[->] (3) to [bend right=40] (2);
\end{tikzpicture}

%% file: figures/multiroot_d.tex
\begin{tikzpicture}[
    every node/.style={
        circle,
        fill=black,
        inner sep=1.5pt
    }
]
    \node[label={below:{\small 0}}] (0) at (0, 0) {};
    \node[label={below:{\small 1}}] (1) at (0.7, 0) {};
    \node[label={below:{\small 2}}] (2) at (1.4, 0) {};
    \node[label={below:{\small 3}}] (3) at (2.1, 0) {};
    
    \draw[->] (0) to [bend left=40] (2);
    \draw[->] (2) to [bend right=40] (1);
    \draw[->] (2) to [bend left=40] (3);
\end{tikzpicture}

%% file: figures/multiroot_e.tex
\begin{tikzpicture}[
    every node/.style={
        circle,
        fill=black,
        inner sep=1.5pt
    }
]
    \node[label={below:{\small 0}}] (0) at (0, 0) {};
    \node[label={below:{\small 1}}] (1) at (0.7, 0) {};
    \node[label={below:{\small 2}}] (2) at (1.4, 0) {};
    \node[label={below:{\small 3}}] (3) at (2.1, 0) {};
    
    \draw[->] (0) to [bend left=40] (2);
    \draw[->] (2) to (1);
    \draw[->] (1) to [bend left=40] (3);
\end{tikzpicture}

%% file: figures/multiroot_f.tex
\begin{tikzpicture}[
    every node/.style={
        circle,
        fill=black,
        inner sep=1.5pt
    }
]
    \node[label={below:{\small 0}}] (0) at (0, 0) {};
    \node[label={below:{\small 1}}] (1) at (0.7, 0) {};
    \node[label={below:{\small 2}}] (2) at (1.4, 0) {};
    \node[label={below:{\small 3}}] (3) at (2.1, 0) {};
    
    \draw[->] (0) to [bend left=40] (2);
    \draw[->] (2) to (3);
    \draw[->] (3) to [bend right=40] (1);
\end{tikzpicture}

%% file: figures/multiroot_g.tex
\begin{tikzpicture}[
    every node/.style={
        circle,
        fill=black,
        inner sep=1.5pt
    }
]
    \node[label={below:{\small 0}}] (0) at (0, 0) {};
    \node[label={below:{\small 1}}] (1) at (0.7, 0) {};
    \node[label={below:{\small 2}}] (2) at (1.4, 0) {};
    \node[label={below:{\small 3}}] (3) at (2.1, 0) {};
    
    \draw[->] (0) to [bend left=40] (3);
    \draw[->] (3) to [bend right=40] (1);
    \draw[->] (3) to [bend right=40] (2);
\end{tikzpicture}

%% file: figures/multiroot_h.tex
\begin{tikzpicture}[
    every node/.style={
        circle,
        fill=black,
        inner sep=1.5pt
    }
]
    \node[label={below:{\small 0}}] (0) at (0, 0) {};
    \node[label={below:{\small 1}}] (1) at (0.7, 0) {};
    \node[label={below:{\small 2}}] (2) at (1.4, 0) {};
    \node[label={below:{\small 3}}] (3) at (2.1, 0) {};
    
    \draw[->] (0) to [bend left=40] (3);
    \draw[->] (3) to [bend right=40] (1);
    \draw[->] (1) to [bend left=40] (2);
\end{tikzpicture}

%% file: figures/multiroot_i.tex
\begin{tikzpicture}[
    every node/.style={
        circle,
        fill=black,
        inner sep=1.5pt
    }
]
    \node[label={below:{\small 0}}] (0) at (0, 0) {};
    \node[label={below:{\small 1}}] (1) at (0.7, 0) {};
    \node[label={below:{\small 2}}] (2) at (1.4, 0) {};
    \node[label={below:{\small 3}}] (3) at (2.1, 0) {};
    
    \draw[->] (0) to [bend left=40] (3);
    \draw[->] (3) to [bend right=40] (2);
    \draw[->] (2) to [bend right=40] (1);
\end{tikzpicture}

%% file: 0.main.bbl
\begin{thebibliography}{35}
\expandafter\ifx\csname natexlab\endcsname\relax\def\natexlab#1{#1}\fi

\bibitem[{Bartlett et~al.(2006)Bartlett, Jordan, and
  McAuliffe}]{bartlett2006convexity}
Peter~L Bartlett, Michael~I Jordan, and Jon~D McAuliffe. 2006.
\newblock Convexity, classification, and risk bounds.
\newblock \emph{Journal of the American Statistical Association},
  101(473):138--156.

\bibitem[{Blondel(2019)}]{blondel2019projection}
Mathieu Blondel. 2019.
\newblock \href
  {https://proceedings.neurips.cc/paper/2019/file/7990ec44fcf3d7a0e5a2add28362213c-Paper.pdf}
  {Structured prediction with projection oracles}.
\newblock In \emph{Advances in Neural Information Processing Systems},
  volume~32. Curran Associates, Inc.

\bibitem[{Blondel et~al.(2020)Blondel, Martins, and Niculae}]{blondel2020fy}
Mathieu Blondel, André~F.T. Martins, and Vlad Niculae. 2020.
\newblock \href {http://jmlr.org/papers/v21/19-021.html} {Learning with
  fenchel-young losses}.
\newblock \emph{Journal of Machine Learning Research}, 21(35):1--69.

\bibitem[{Chu and Liu(1965)}]{chu1965msa}
Yoeng-Jin Chu and Tseng-Hong Liu. 1965.
\newblock On the shortest arborescence of a directed graph.
\newblock \emph{Scientia Sinica}.

\bibitem[{Corro(2020)}]{corro2020disc}
Caio Corro. 2020.
\newblock \href {https://doi.org/10.18653/v1/2020.emnlp-main.219} {Span-based
  discontinuous constituency parsing: a family of exact chart-based algorithms
  with time complexities from {O}(n{\^{}}6) down to {O}(n{\^{}}3)}.
\newblock In \emph{Proceedings of the 2020 Conference on Empirical Methods in
  Natural Language Processing (EMNLP)}, pages 2753--2764, Online. Association
  for Computational Linguistics.

\bibitem[{Duchi et~al.(2018)Duchi, Khosravi, and Ruan}]{duchi2018}
John Duchi, Khashayar Khosravi, and Feng Ruan. 2018.
\newblock \href {https://doi.org/10.1214/17-AOS1657} {{Multiclass
  classification, information, divergence and surrogate risk}}.
\newblock \emph{The Annals of Statistics}, 46(6B):3246 -- 3275.

\bibitem[{Edmonds(1967)}]{edmonds1967msa}
Jack Edmonds. 1967.
\newblock Optimum branchings.
\newblock \emph{Journal of Research of the National Bureau of Standards -- B.
  Mathematics and Mathematical Physics}.

\bibitem[{Effland and Collins(2021)}]{effland2021partially}
Thomas Effland and Michael Collins. 2021.
\newblock \href {https://doi.org/10.1162/tacl_a_00429} {Partially supervised
  named entity recognition via the expected entity ratio loss}.
\newblock \emph{Transactions of the Association for Computational Linguistics},
  9:1320--1335.

\bibitem[{Fischetti and Toth(1992)}]{fischetti1992additive}
Matteo Fischetti and Paolo Toth. 1992.
\newblock An additive bounding procedure for the asymmetric travelling salesman
  problem.
\newblock \emph{Mathematical Programming}, 53(1):173--197.

\bibitem[{Gneiting and Raftery(2007)}]{tilmann2007}
Tilmann Gneiting and Adrian~E Raftery. 2007.
\newblock \href {https://doi.org/10.1198/016214506000001437} {Strictly proper
  scoring rules, prediction, and estimation}.
\newblock \emph{Journal of the American Statistical Association},
  102(477):359--378.

\bibitem[{Koo et~al.(2007)Koo, Globerson, Carreras, and Collins}]{koo2007mtt}
Terry Koo, Amir Globerson, Xavier Carreras, and Michael Collins. 2007.
\newblock \href {https://aclanthology.org/D07-1015} {Structured prediction
  models via the matrix-tree theorem}.
\newblock In \emph{Proceedings of the 2007 Joint Conference on Empirical
  Methods in Natural Language Processing and Computational Natural Language
  Learning ({EMNLP}-{C}o{NLL})}, pages 141--150, Prague, Czech Republic.
  Association for Computational Linguistics.

\bibitem[{Lafferty et~al.(2001)Lafferty, McCallum, and
  Pereira}]{lafferty2001crf}
John~D. Lafferty, Andrew McCallum, and Fernando C.~N. Pereira. 2001.
\newblock Conditional random fields: Probabilistic models for segmenting and
  labeling sequence data.
\newblock In \emph{Proceedings of the Eighteenth International Conference on
  Machine Learning {(ICML} 2001), Williams College, Williamstown, MA, USA, June
  28 - July 1, 2001}, pages 282--289. Morgan Kaufmann.

\bibitem[{Lin(2004)}]{lin2004fisher}
Yi~Lin. 2004.
\newblock A note on margin-based loss functions in classification.
\newblock \emph{Statistics \& probability letters}, 68(1):73--82.

\bibitem[{Liu(2007)}]{liu07consistency}
Yufeng Liu. 2007.
\newblock \href {https://proceedings.mlr.press/v2/liu07b.html} {Fisher
  consistency of multicategory support vector machines}.
\newblock In \emph{Proceedings of the Eleventh International Conference on
  Artificial Intelligence and Statistics}, volume~2 of \emph{Proceedings of
  Machine Learning Research}, pages 291--298, San Juan, Puerto Rico. PMLR.

\bibitem[{Long and Servedio(2013)}]{dasgupta2013hconsistency}
Phil Long and Rocco Servedio. 2013.
\newblock \href {https://proceedings.mlr.press/v28/long13.html} {Consistency
  versus realizable h-consistency for multiclass classification}.
\newblock In \emph{Proceedings of the 30th International Conference on Machine
  Learning}, volume~28 of \emph{Proceedings of Machine Learning Research},
  pages 801--809, Atlanta, Georgia, USA. PMLR.

\bibitem[{Lugosi and Vayatis(2004)}]{gabor2004consistency}
G{\'a}bor Lugosi and Nicolas Vayatis. 2004.
\newblock \href {https://doi.org/10.1214/aos/1079120129} {{On the Bayes-risk
  consistency of regularized boosting methods}}.
\newblock \emph{The Annals of Statistics}, 32(1):30 -- 55.

\bibitem[{Ma and Collins(2018)}]{ma2018nce}
Zhuang Ma and Michael Collins. 2018.
\newblock \href {https://doi.org/10.18653/v1/D18-1405} {Noise contrastive
  estimation and negative sampling for conditional models: Consistency and
  statistical efficiency}.
\newblock In \emph{Proceedings of the 2018 Conference on Empirical Methods in
  Natural Language Processing}, pages 3698--3707, Brussels, Belgium.
  Association for Computational Linguistics.

\bibitem[{McDonald et~al.(2005)McDonald, Pereira, Ribarov, and
  Haji{\v{c}}}]{mcdonald2005msa}
Ryan McDonald, Fernando Pereira, Kiril Ribarov, and Jan Haji{\v{c}}. 2005.
\newblock \href {https://aclanthology.org/H05-1066} {Non-projective dependency
  parsing using spanning tree algorithms}.
\newblock In \emph{Proceedings of Human Language Technology Conference and
  Conference on Empirical Methods in Natural Language Processing}, pages
  523--530, Vancouver, British Columbia, Canada. Association for Computational
  Linguistics.

\bibitem[{McDonald and Satta(2007)}]{mcdonald2007mtt}
Ryan McDonald and Giorgio Satta. 2007.
\newblock \href {https://aclanthology.org/W07-2216} {On the complexity of
  non-projective data-driven dependency parsing}.
\newblock In \emph{Proceedings of the Tenth International Conference on Parsing
  Technologies}, pages 121--132, Prague, Czech Republic. Association for
  Computational Linguistics.

\bibitem[{Nowak et~al.(2019)Nowak, Bach, and Rudi}]{nowak2019general}
Alex Nowak, Francis Bach, and Alessandro Rudi. 2019.
\newblock A general theory for structured prediction with smooth convex
  surrogates.
\newblock \emph{arXiv preprint arXiv:1902.01958}.

\bibitem[{Nowak et~al.(2020)Nowak, Bach, and Rudi}]{nowak20maxminmargin}
Alex Nowak, Francis Bach, and Alessandro Rudi. 2020.
\newblock \href {https://proceedings.mlr.press/v119/nowak20a.html} {Consistent
  structured prediction with max-min margin {M}arkov networks}.
\newblock In \emph{Proceedings of the 37th International Conference on Machine
  Learning}, volume 119 of \emph{Proceedings of Machine Learning Research},
  pages 7381--7391. PMLR.

\bibitem[{Nowak et~al.(2022)Nowak, Rudi, and Bach}]{nowak22maxmargin}
Alex Nowak, Alessandro Rudi, and Francis Bach. 2022.
\newblock \href {https://proceedings.mlr.press/v151/nowak22a.html} {On the
  consistency of max-margin losses}.
\newblock In \emph{Proceedings of The 25th International Conference on
  Artificial Intelligence and Statistics}, volume 151 of \emph{Proceedings of
  Machine Learning Research}, pages 4612--4633. PMLR.

\bibitem[{Pasupat et~al.(2019)Pasupat, Gupta, Mandyam, Shah, Lewis, and
  Zettlemoyer}]{pasupat2019top}
Panupong Pasupat, Sonal Gupta, Karishma Mandyam, Rushin Shah, Mike Lewis, and
  Luke Zettlemoyer. 2019.
\newblock \href {https://doi.org/10.18653/v1/D19-1163} {Span-based hierarchical
  semantic parsing for task-oriented dialog}.
\newblock In \emph{Proceedings of the 2019 Conference on Empirical Methods in
  Natural Language Processing and the 9th International Joint Conference on
  Natural Language Processing (EMNLP-IJCNLP)}, pages 1520--1526, Hong Kong,
  China. Association for Computational Linguistics.

\bibitem[{Ratinov and Roth(2009)}]{ratinov2009bilou}
Lev Ratinov and Dan Roth. 2009.
\newblock \href {https://aclanthology.org/W09-1119} {Design challenges and
  misconceptions in named entity recognition}.
\newblock In \emph{Proceedings of the Thirteenth Conference on Computational
  Natural Language Learning ({C}o{NLL}-2009)}, pages 147--155, Boulder,
  Colorado. Association for Computational Linguistics.

\bibitem[{Reid and Williamson(2010)}]{reid10}
Mark~D. Reid and Robert~C. Williamson. 2010.
\newblock \href {http://jmlr.org/papers/v11/reid10a.html} {Composite binary
  losses}.
\newblock \emph{Journal of Machine Learning Research}, 11(83):2387--2422.

\bibitem[{Smith and Smith(2007)}]{smith2007mtt}
David~A. Smith and Noah~A. Smith. 2007.
\newblock \href {https://aclanthology.org/D07-1014} {Probabilistic models of
  nonprojective dependency trees}.
\newblock In \emph{Proceedings of the 2007 Joint Conference on Empirical
  Methods in Natural Language Processing and Computational Natural Language
  Learning ({EMNLP}-{C}o{NLL})}, pages 132--140, Prague, Czech Republic.
  Association for Computational Linguistics.

\bibitem[{Tarjan(1977)}]{tarjan1977msa}
Robert~Endre Tarjan. 1977.
\newblock Finding optimum branchings.
\newblock \emph{Networks}, 7(1):25--35.

\bibitem[{Tewari and Bartlett(2007)}]{tewari07consistency}
Ambuj Tewari and Peter~L. Bartlett. 2007.
\newblock \href {http://jmlr.org/papers/v8/tewari07a.html} {On the consistency
  of multiclass classification methods}.
\newblock \emph{Journal of Machine Learning Research}, 8(36):1007--1025.

\bibitem[{Titsias(2016)}]{titsias2016onevseach}
Michalis Titsias. 2016.
\newblock \href
  {https://proceedings.neurips.cc/paper/2016/file/814a9c18f5abff398787c9cfcbf3d80c-Paper.pdf}
  {One-vs-each approximation to softmax for scalable estimation of
  probabilities}.
\newblock In \emph{Advances in Neural Information Processing Systems},
  volume~29. Curran Associates, Inc.

\bibitem[{Wainwright and Jordan(2008)}]{wainwright2008exp}
Martin~J. Wainwright and Michael~Irwin Jordan. 2008.
\newblock \emph{Graphical models, exponential families, and variational
  inference}.
\newblock Now Publishers Inc.

\bibitem[{Williamson et~al.(2016)Williamson, Vernet, and Reid}]{williamson2016}
Robert~C. Williamson, Elodie Vernet, and Mark~D. Reid. 2016.
\newblock \href {http://jmlr.org/papers/v17/14-294.html} {Composite multiclass
  losses}.
\newblock \emph{Journal of Machine Learning Research}, 17(222):1--52.

\bibitem[{Zhang(2004{\natexlab{a}})}]{zhang2004statistical}
Tong Zhang. 2004{\natexlab{a}}.
\newblock Statistical analysis of some multi-category large margin
  classification methods.
\newblock \emph{Journal of Machine Learning Research}, 5(Oct):1225--1251.

\bibitem[{Zhang(2004{\natexlab{b}})}]{zhang2004}
Tong Zhang. 2004{\natexlab{b}}.
\newblock \href {https://doi.org/10.1214/aos/1079120130} {{Statistical behavior
  and consistency of classification methods based on convex risk
  minimization}}.
\newblock \emph{The Annals of Statistics}, 32(1):56 -- 85.

\bibitem[{Zhang et~al.(2017)Zhang, Cheng, and Lapata}]{zhang2017head}
Xingxing Zhang, Jianpeng Cheng, and Mirella Lapata. 2017.
\newblock \href {https://aclanthology.org/E17-1063} {Dependency parsing as head
  selection}.
\newblock In \emph{Proceedings of the 15th Conference of the {E}uropean Chapter
  of the Association for Computational Linguistics: Volume 1, Long Papers},
  pages 665--676, Valencia, Spain. Association for Computational Linguistics.

\bibitem[{Zhang et~al.(2020)Zhang, Li, and Zhang}]{zhang2020treecrf}
Yu~Zhang, Zhenghua Li, and Min Zhang. 2020.
\newblock \href {https://doi.org/10.18653/v1/2020.acl-main.302} {Efficient
  second-order {T}ree{CRF} for neural dependency parsing}.
\newblock In \emph{Proceedings of the 58th Annual Meeting of the Association
  for Computational Linguistics}, pages 3295--3305, Online. Association for
  Computational Linguistics.

\end{thebibliography}
